\documentclass{article}

\usepackage{microtype}
\usepackage{graphicx}
\usepackage{subcaption}
\usepackage{booktabs} 
\usepackage{tcolorbox}
\usepackage{tabularx}
\usepackage{hyperref}

\usepackage[preprint]{icml2026}

\usepackage{amsmath}
\usepackage{amssymb}
\usepackage{multirow}
\usepackage{mathtools}
\usepackage{amsthm}
\usepackage[table]{xcolor}

\usepackage[capitalize,noabbrev]{cleveref}

\theoremstyle{plain}
\newtheorem{theorem}{Theorem}[section]

\theoremstyle{definition}
\newtheorem{definition}[theorem]{Definition}

\theoremstyle{remark}
\newtheorem{remark}[theorem]{Remark}

\usepackage[textsize=tiny]{todonotes}

\icmltitlerunning{Beyond Statistical Co-occurrence: Unlocking Intrinsic Semantics for Tabular Data Clustering}

\begin{document}

\twocolumn[
  \icmltitle{ Beyond Statistical Co-occurrence: Unlocking Intrinsic \\ Semantics for Tabular Data Clustering}
  %



  \icmlsetsymbol{equal}{*}

  \begin{icmlauthorlist}
    \icmlauthor{Mingjie Zhao}{yyy}
    \icmlauthor{Yunfan Zhang}{yyy}
    \icmlauthor{Yiqun Zhang}{comp}
    \icmlauthor{Yiu-ming Cheung*}{yyy}
  \end{icmlauthorlist}

  \icmlaffiliation{yyy}{Department of Computer Science, Hong Kong Baptist University, Hong Kong, China,}
  \icmlaffiliation{comp}{Department of Computer Science, Guangdong University of Technology, Guangzhou, China}

  \icmlcorrespondingauthor{Yiu-ming Cheung}{ymc@comp.hkbu.edu.hk}

  \icmlkeywords{Machine Learning, ICML}

  \vskip 0.3in
]



\printAffiliationsAndNotice{}  

\begin{abstract}

Deep Clustering (DC) has emerged as a powerful tool for tabular data analysis in real-world domains like finance and healthcare. However, most existing methods rely on data-level statistical co-occurrence to infer the latent metric space, often overlooking the intrinsic semantic knowledge encapsulated in feature names and values. As a result, semantically related concepts like `Flu' and `Cold' are often treated as symbolic tokens, causing conceptually related samples to be isolated. To bridge the gap between dataset-specific statistics and intrinsic semantic knowledge, this paper proposes Tabular-Augmented Contrastive Clustering (TagCC), a novel framework that anchors statistical tabular representations to open-world textual concepts. Specifically, TagCC utilizes Large Language Models (LLMs) to distill underlying data semantics into textual anchors via semantic-aware transformation. Through Contrastive Learning (CL), the framework enriches the statistical tabular representations with the open-world  semantics encapsulated in these anchors. This CL framework is jointly optimized with a clustering objective, ensuring that the learned representations are both semantically coherent and clustering-friendly. Extensive experiments on benchmark datasets demonstrate that TagCC significantly outperforms its counterparts.

\end{abstract}

\section{Introduction}
\label{sct:intro}

\begin{figure}[!t]	
\centering
\centerline{\includegraphics[width=3.2in]{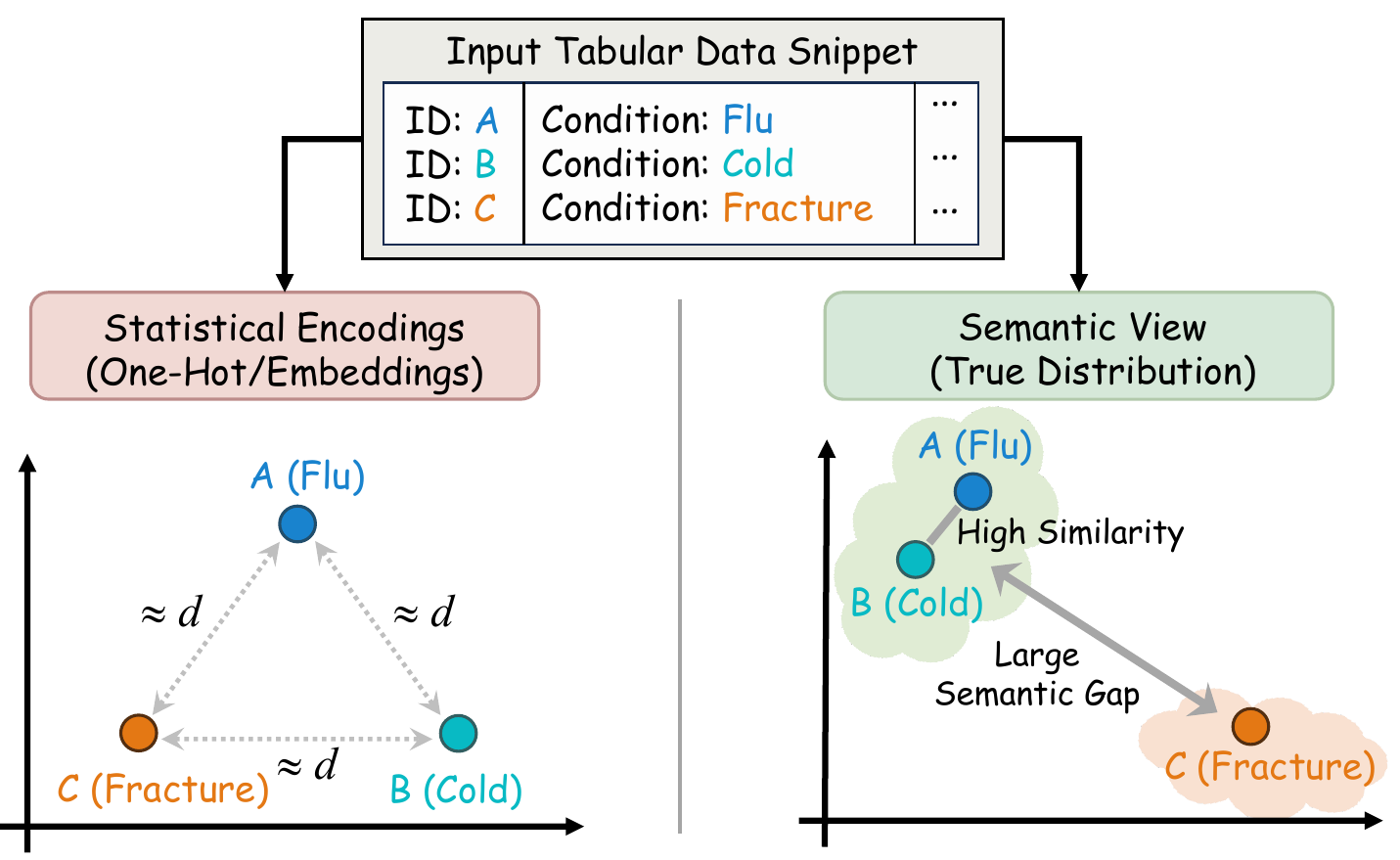}}
\caption{Comparison between statistical encoding (Left) and semantic relationship (Right). Statistical encoding strategies treat related concepts like `Flu' and `Cold' as orthogonal to `Fracture', i.e., equidistant. In contrast, ideally, the latent space should be aligned with open-world  semantics to ensure that conceptually similar samples are partitioned into coherent clusters.}
\label{fig:pre_ex}	
\end{figure}

Deep Clustering (DC) has emerged as a promising technique for knowledge discovery, particularly for tabular datasets widespread in real-world domains like finance and healthcare~\cite{survey6, NatureTableModel}. Unlike traditional shallow models~\cite{ADC,HARR}, DC leverages the powerful non-linear fitting capabilities of neural networks to capture intricate feature interactions within raw data, mapping them into a latent representation where clustering objectives can be optimized~\cite{survey4,survey5,RDLTD}. To effectively learn such representations, existing frameworks commonly rely on two paradigms: ReConstruction-based (RC) methods~\cite{IDC,NeuDCTD,IDEC}, which learn representations by minimizing data reconstruction error to preserve local structure, and Contrastive Learning-based (CL) methods~\cite{dmcc, wang2020understanding}, which maximize the agreement between augmented views of the same sample while pushing apart distinct samples to learn robust representations.

These approaches have demonstrated significant success
in capturing the statistical correlations and data manifolds
within tabular data. Nevertheless, they are limited to inferring latent dependencies implicitly through data-level co-occurrences, such as statistical associations between features, treating feature values merely as symbolic tokens or discrete scalars while disregarding the semantic priors encapsulated in tabular schemas, e.g., feature names and context of feature values. In the real world, where samples are often limited and sparse, such reliance is insufficient to support valid statistical inference. As illustrated in Figure~\ref{fig:pre_ex}, solely relying on such sparse statistical signals may fail to capture the proximity between conceptually related terms like `Flu' and `Cold', leaving them as orthogonal as unrelated concepts like `Fracture', thereby limiting the model's capacity to cluster similar samples.

Recently, some studies have attempted to incorporate explicit semantic priors into tabular learning by leveraging Pre-trained Language Models (PLM)~\cite{tabllm}. They typically employ serialization strategies~\cite{hegselmann2023tabllm}, e.g., concatenation of feature names and values, to convert table rows into text sequences for representation learning~\cite{survey2,survey1}. However, such concatenation-based serialization often overlook the distance structure that is essential for clustering, resulting in verbose textual descriptions that lack discriminative focus for partitioning tabular samples~\cite{transtab,iida2021tabbie}. To mitigate this deficiency, methods like~\cite{tabledc} derive initial embeddings from serialized text and subsequently constrain them using statistical priors, such as Mahalanobis distance. Despite these efforts to introduce statistical guidance, the model inevitably inherits the vulnerability to data sparsity. Even in large datasets, tabular data manifest as sparse, discrete points rather than continuous manifolds~\cite{survey2,survey1}. Consequently, heavy reliance on statistical priors drives the model to overfit isolated statistical correlations inferred from these discrete points, impeding the learning of generalized semantic representations. As a result, the intrinsic tabular semantics, weakened by concatenation-based serialization, are further overshadowed during training. This imposes an inevitable performance ceiling, as the clustering objective cannot compensate for the under-utilized semantic priors and unreliable statistical priors.

In this paper, we propose Tabular-Augmented Contrastive Clustering (TagCC), a novel framework that anchors statistical tabular representations to open-world textual concepts derived from tabular schemas. Specifically, distinct from relying on concatenation-based serialization, we leverage the extensive open-world knowledge of Large Language Models (LLMs) to interpret tabular semantics. To this end, we devise a semantic-aware transformation strategy. This mechanism explicitly guides LLM to distill intrinsic data semantics into structured, context-enriched textual anchors, ensuring the generated descriptions remain discriminative for clustering. By aligning tabular representations with these anchors via CL, TagCC enriches the statistical latent space with open-world semantics while preventing the representations from overfitting to isolated statistical correlations. Simultaneously, we jointly optimized the CL objective with the clustering objective, ensuring that the resulting partitions are both semantically coherent and structurally compact. Our main contributions are summarized as follows:

\begin{itemize}
    \item We propose the TagCC framework that introduces a semantic anchoring paradigm to tabular learning. Unlike existing methods that rely on statistical co-occurrence to infer latent metric space, TagCC anchors the tabular representations to stable open-world textual concepts, incorporating semantic priors to compensate for data sparsity and effectively preventing the model from overfitting to isolated statistical correlations.
    
    \item A semantic-aware transformation strategy coupled with a CL mechanism is proposed. These components encode intrinsic data semantics into context-enriched textual anchors and explicitly align tabular representations with these anchors, ensuring that the learned tabular representations capture open-world semantics and are discriminative for clustering.

    \item Extensive experiments on multiple benchmark datasets demonstrate that TagCC achieves superior clustering performance against state-of-the-art counterparts. This highlights that integrating open-world semantic priors yields semantically coherent partitions, outperforming methods that are constrained to inferring data distributions from closed statistical co-occurrences.
\end{itemize}

\begin{figure*}[!t]	
\centering
\centerline{\includegraphics[width=6.3in]{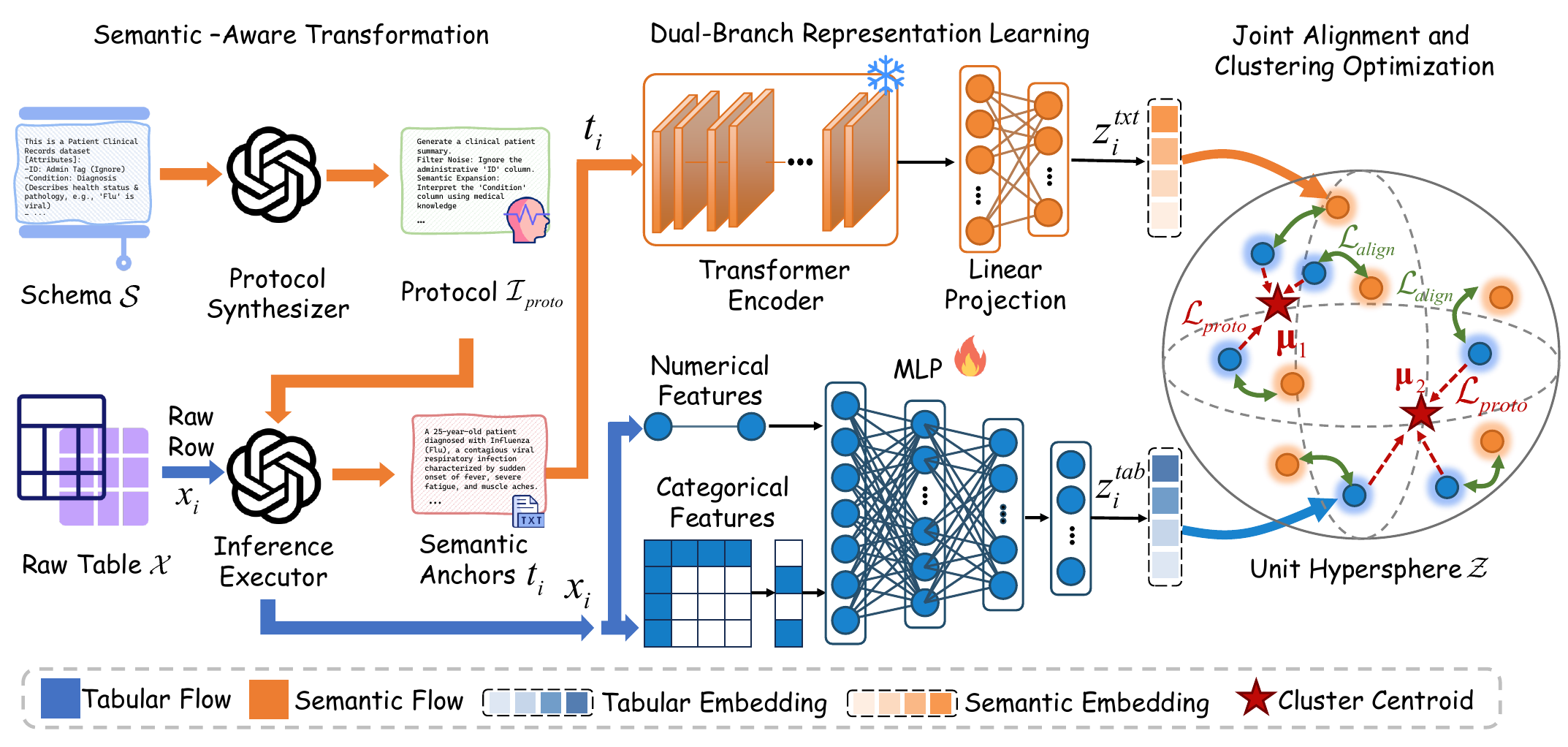}}
\caption{Overview of the TagCC framework. To bridge the semantic gap in tabular clustering, the framework augments raw table rows with LLM-synthesized semantic anchors $t_i$. These anchors and raw features $\mathbf{x}_i$ are then processed by a dual-branch encoder, where the trainable tabular network is optimized to align with the frozen semantic backbone. The optimization is conducted on a unit hypersphere $\mathcal{Z}$, where the model jointly minimizes CL loss $\mathcal{L}_{align}$ for semantic consistency and DC loss $\mathcal{L}_{proto}$ for structural compactness.
}
\label{fig:framework}
\end{figure*}

\section{Related Work}
\label{sct:rw}

In this section, we provide an overview of representation learning for clustering, categorizing methodologies into \textbf{RC-based} and \textbf{CL-based} paradigms, and recent methods for learning with \textbf{LLMs for tabular data}.

\textbf{RC-based Methods.} Early methods like~\cite{DEC} and~\cite{IDEC} optimize clustering by preserving local structure via autoencoder reconstruction constraints, while variational approaches~\cite{vade,dgg} extend this to probabilistic frameworks. Recent works such as~\cite{QGRL} and~\cite{IDC} have refined these for tabular data by modeling heterogeneous feature interactions. However, they fundamentally rely on geometric proximity, ignoring intrinsic semantics. Although~\cite{tabledc} attempts to incorporate semantics via serialized text initialization, it suffers from semantic forgetting. As its optimization is governed by statistical constraints (e.g., Mahalanobis distance), the initial open-world knowledge is diluted during training, causing the model to revert to local statistical dependencies.

\textbf{CL-based Methods.} This paradigm extracts discriminative features by maximizing agreement between augmented views. Methods like~\cite{scarf} and~\cite{subtab} enforce instance-level consistency, while graph-based approaches~\cite{sdcn,GCN} leverage GCNs to capture topological structures. Despite these designs, positive pairs are defined solely by intra-dataset statistical co-occurrence. Treating values as symbolic tokens, these methods overlook the semantic proximity between concepts, resulting in ambiguous cluster boundaries when statistical features are sparse.

\textbf{LLMs for Tabular Data}. Recent adaptations like~\cite{transtab} and~\cite{NatureTableModel} utilize transformer architectures to model field interactions from serialized text. However, these frameworks are predominantly designed for supervised tasks, relying on labels to resolve ambiguity. Directly applying them to clustering yields suboptimal results, as generic textual representations lack the specific cluster-structure constraints required for partitioning data.

\section{Preliminaries}
\label{sct:prelim}
In this section, we introduce the notations and formally define the problem of deep tabular clustering.

Consider a tabular dataset $X=\{\mathbf{x}_1,\mathbf{x}_2,\ldots,\mathbf{x}_n\}$ with $n$ data samples. Each data sample can be denoted as an $m$-dimensional row vector $\mathbf{x}_i=[x_{i,1},x_{i,2},\ldots,x_{i,m}]^\top$ represented by $m$ features $A=\{\mathbf{a}_1,\mathbf{a}_2,...,\mathbf{a}_m\}$. The dataset is structured by a schema $\mathcal{S}$, which comprises comprehensive metadata, e.g., background context of the dataset, feature definitions, and specific interpretations of feature values. The goal of deep tabular clustering is to jointly learn a non-linear mapping $f_\theta: \mathcal{X} \rightarrow \mathcal{Z}$ and a set of trainable centroids $\mathcal{M} = \{\boldsymbol{\mu}_j\}_{j=1}^k$, effectively projecting the raw data into a lower-dimensional latent space $\mathcal{Z} \in \mathbb{R}^d$ organized into $k$ disjoint clusters $\mathcal{C} = \{C_1, \dots, C_k\}$. 

This paper focuses on aligning the representation learning with semantic structures to facilitate semantically coherent clustering. Therefore, the key lies in how to effectively leverage the static metadata in $\mathcal{S}$ to generate discriminative semantic anchors for the tabular data, explicitly guiding the optimization of $f_\theta$ to yield representations that are both statistically compact and semantically distinctive.

\section{Proposed Method}
\label{sct:pm}

We first detail the semantic-aware transformation in Section~\ref{sct:transformation}, where LLMs are leveraged to distill intrinsic semantics into textual anchors. Then, in Section~\ref{sct:dual_view}, we introduce the dual-branch representation learning to explicitly align statistical representations with these semantic priors. Finally, the joint optimization objective is presented in Section~\ref{sct:optimization}. The pipeline of TagCC is shown in Figure~\ref{fig:framework}.

\subsection{Semantic-Aware Transformation}
\label{sct:transformation}

To synthesize semantic anchors that are both semantically enriched and discriminative for clustering, avoiding the information degradation inherent in serialization, a guided semantic inference process is essential. However, directly prompting LLMs to interpret tabular rows often yields overly generic outputs due to the lack of task-specific constraints. To mitigate this, we employ a self-instruction strategy that empowers the LLM to assume dual roles, acting as a protocol synthesizer to derive reasoning standards and an inference executor to perform guided transformation. Consequently, the implementation hinges on addressing two fundamental tasks: \textit{T1:} automatically distilling a transformation protocol from the tabular schema; and \textit{T2:} leveraging this protocol to synthesize discriminative, semantic-aware textual anchors.

For \textit{T1}, the LLM acting as a protocol synthesizer to analyzes the tabular schema $\mathcal{S}$ to synthesize a dataset-specific transformation protocol. We systematically provide the generation process with justifications as follows:

\begin{definition}[\textbf{semantic-aware Meta-Principles} $\mathcal{P}_{meta}$]
   Let $\mathcal{P}_{meta} = \{ \psi_{\text{hor}}, \psi_{\text{ver}} \}$ be a tuple of constraints that conditions the instruction generation on a schema-derived analytic perspective, e.g., an economist classifying income:
    \textbf{1) Horizontal Salience ($\psi_{\text{hor}}$):} The instruction should maximize information density by filtering out administrative noise irrelevant to the analytic goal; \textbf{2) Vertical Distinctiveness ($\psi_{\text{ver}}$):} The instruction must maximize semantic separability by interpreting symbolic terms into open-world concepts from the domain expert's perspective.
\end{definition}

Guided by these principles, the LLM formulates the schema $\mathcal{S}$ into a dataset-specific transformation protocol $\mathcal{I}_{proto}$
\begin{equation}
    \label{eq:gen_task}
    \mathcal{I}_{proto} = \text{LLM}(\mathcal{S}, \mathcal{P}_{meta}).
\end{equation}
This automated instantiation distills the tabular schema into executable inference rules, establishing a unified reasoning logic that governs the subsequent interpretation.

For \textit{T2}, the transformation is formulated as a protocol-guided inference task~\cite{TLhf}. Specifically, the LLM interprets the raw data $\mathbf{x}_i$ by strictly adhering to the reasoning logic defined in the derived protocol $\mathcal{I}_{proto}$:
\begin{equation}
    \label{eq:gen_tex}
    t_i \sim P_{\text{LLM}}(t \mid \mathbf{x}_i, \mathcal{I}_{proto}).
\end{equation}
By adhering to $\mathcal{I}_{proto}$, this controlled reasoning process generates a semantic-aware representation $t_i$ that encodes intrinsic data semantics while preserving structural distinctiveness for each sample. These representations $\{t_i\}_{i=1}^n$ serve as stable semantic anchors and are subsequently leveraged to guide the tabular representation learning.

\subsection{Dual-Branch Representation Learning}
\label{sct:dual_view}

After obtaining the textual anchors $\{t_i\}_{i=1}^n$, our goal is to utilize them to explicitly guide the tabular encoder $f_\theta$, to compensate for the sparsity of statistical information. To facilitate this, it is essential to align the feature spaces of the tabular data and the textual anchors to enable effective semantic injection. To this end, we propose a dual-branch representation learning framework. Specifically, this architecture deploys a tabular encoder branch and a semantic adapter branch to project these disparate modalities into a unified metric space, enabling explicit semantic alignment.

\textbf{Tabular encoder branch.} To distill intrinsic statistical patterns from the raw sample $\mathbf{x}_i$, the $f_\theta$ is designed to handle the inherent heterogeneity of tabular data. Unlike homogeneous pixel inputs, tabular data typically comprises mixed continuous features $\mathbf{a}_{num}$ and discrete features $\mathbf{a}_{cat}$. To map these disparate features into a compatible latent space, a type-specific embedding strategy is employed. Specifically, guided by schema $\mathcal{S}$, the $l$-th feature of $x_{i}$ is projected into a unified dense vector $e_{i,l}$:
\begin{equation}
    \label{eq:compute_emb}
    e_{i,l} = 
    \begin{cases} 
        \mathbf{W}_l x_{i,l} + \mathbf{b}_l, & \text{if } \mathbf{a}_{l} \in \mathbf{a}_{num} \\
        \text{Emb}_l(x_{i,l}), & \text{if } \mathbf{a}_{l} \in \mathbf{a}_{cat}
    \end{cases},
\end{equation}
where $l=\{1, \dots, m\}$. Here, $\mathbf{W}_l, \mathbf{b}_l$ and $\text{Emb}_l(\cdot)$ denote the learnable projection parameters for numerical and categorical features, respectively. Subsequently, to model complex high-order feature interactions, these local embeddings are concatenated and fused via a Multi-Layer Perceptron (MLP) $\Phi_{tab}$, formally defining the encoder mapping $f_\theta$:
\begin{equation}
    \label{eq:gen_rep_ta}
    f_\theta(\mathbf{x}_i) = \Phi_{tab}\left( \left[ e^{(1)} \oplus e^{(2)} \oplus \dots \oplus e^{(m)} \right] \right),
\end{equation}
The resulting $f_\theta(\mathbf{x}_i)$ serves as a statistically rich representation encoding both feature-specific properties and their non-linear correlations.

\textbf{Semantic adapter branch.} In parallel, to incorporate open-world knowledge into the clustering framework, the auxiliary semantic adapter $g_{\phi}$ processes the textual anchor $t_i$ generated in the previous stage. To align general semantic knowledge with the task-specific subspace, the text is first encoded by a frozen PLM backbone $E_{plm}$ (e.g., Sentence-BERT) and then transformed by a learnable linear projector:
\begin{equation}
    \label{eq:gen_rep_tx}
    g_{\phi}(t_i) = \mathbf{W}_{sem} E_{plm}(t_i) + \mathbf{b}_{sem},
\end{equation}
where $\mathbf{W}_{sem}$ and $\mathbf{b}_{sem}$ are the projection parameters.


\begin{remark}[Semantic Anchoring Strategy]
  
The frozen semantic backbone $E_{plm}$ acts as a semantic manifold regularizer to prevent representational co-adaptation and mode collapse. By providing a stationary target distribution, this strategy forces the tabular encoder to converge towards the pre-defined semantic structure, avoiding the trivial solutions common in unconstrained joint optimization.

\end{remark}

\textbf{Unified project head.} To bridge the distribution gap between the tabular and texture modalities, the raw representation $f_\theta(\mathbf{x}_i)$ and the semantic representation $g_\phi(t_i)$ are mapped onto a shared unit hypersphere $\mathcal{Z}$. Distinct non-linear projection heads ($\text{Proj}_{\theta}$ and $\text{Proj}_{\phi}$) are applied, followed by $\ell_2$ normalization:
\begin{equation}
    \label{eq:comput_z}
    z_i^{tab} = \frac{\text{Proj}_{\theta}(f_\theta(\mathbf{x}_i))}{\|\text{Proj}_{\theta}(f_\theta(\mathbf{x}_i))\|_2}, \quad z_i^{txt} = \frac{\text{Proj}_{\phi}(g_\phi(t_i))}{\|\text{Proj}_{\phi}(g_\phi(t_i))\|_2}.
\end{equation}
Here, $z_i^{tab}, z_i^{txt} \in \mathbb{R}^d$ denote the final normalized embeddings. This geometric alignment ensures that the dot product directly reflects semantic similarity, thereby facilitating the subsequent semantic-aware clustering objective.

\subsection{Joint Alignment and Clustering Optimization}
\label{sct:optimization}
After obtaining the unified representations for tabular data and semantic anchors, to achieve semantically coherent clustering and avoid overfitting to isolated statistical correlations, semantic consistency and structural compactness are simultaneously enforced within the shared latent space. Accordingly, the optimization is governed by a tripartite objective function, integrating cross-view CL, prototype-based DC, and entropy regularization:
\begin{equation}
    \label{eq:all_lost}
    \mathcal{L}_{total} = \mathcal{L}_{align} + \alpha\mathcal{L}_{proto} + \beta\mathcal{L}_{ent},
\end{equation}
where $\alpha$ and $\beta$ are balancing coefficients.

To align the tabular representation with open-world semantics, the generated anchor $z_i^{txt}$ serves as the positive reference for the tabular feature $z_i^{tab}$. A CL loss is employed to maximize the agreement between corresponding representations within the batch of size $B$:
\begin{equation}
    \label{eq:loss_align}
    \mathcal{L}_{align} = - \frac{1}{B} \sum_{i=1}^B \log \frac{\exp(\langle z_i^{tab}, z_i^{txt} \rangle / \tau)}{\sum_{j=1}^B \exp(\langle z_i^{tab}, z_j^{txt} \rangle / \tau)},
\end{equation}
where $\tau$ is the temperature parameter and $\langle \cdot, \cdot \rangle$ denotes the dot product. This alignment establishes a semantically meaningful layout in $\mathcal{Z}$ before cluster structures are explicitly formed.

\begin{theorem}[Robustness against hallucination]
\label{thm:robustness_main}

The contrastive alignment objective $\mathcal{L}_{align}$ acts as a robust regularizer against semantic hallucination. It effectively limits the impact of suboptimal anchors on the learning process, ensuring that the model's error grows much slower than the hallucination magnitude $\epsilon$ at a sub-linear rate of $\mathcal{O}(\sqrt{\epsilon})$.

\end{theorem}

\begin{proof}
    See Theorem~\ref{thm:robustness} in Appendix \ref{sct:robustness} and Appendix~\ref{sct:exp_huRobust} for the proofs and empirical evidence.
\end{proof}
 
To uncover the latent cluster structure within the aligned space, the framework employs a set of learnable centroids $\mathcal{M} = \{\boldsymbol{\mu}_j\}_{j=1}^k$. Following the~\cite{li2021prototypical}, the probability $p_{ij}$ of assigning sample $z_i^{tab}$ to cluster $C_j$ is computed via a softmax over the similarity scores:
\begin{equation}
    \label{eq:compute_assign}
    p_{ij} = \frac{\exp(\langle z_i^{tab}, \boldsymbol{\mu}_j \rangle / \tau)}{\sum_{j'=1}^k \exp(\langle z_i^{tab}, \boldsymbol{\mu}_{j'} \rangle / \tau)}.
\end{equation}
Let $P \in \mathbb{R}^{B \times k}$ denote the soft assignment matrix with entries $p_{ij}$. To enforce high-confidence assignments, a target distribution matrix $Q \in \mathbb{R}^{B \times k}$ is generated, where each entry $q_{ij}$ is obtained by sharpening $p_{ij}$:
    \begin{equation}
        \label{eq:compu_q}
        q_{ij} = \frac{p_{ij}^2 / \sum_i p_{ij}}{\sum_{j'=1}^k (p_{ij'}^2 / \sum_i p_{ij'})}.
    \end{equation}
The clustering loss then minimizes the Kullback-Leibler divergence between the target $Q$ and the prediction $P$:
\begin{equation}
    \label{eq:comp_proto_loss}
    \mathcal{L}_{proto} = \text{KL}(Q \| P) = \sum_{i=1}^B \sum_{j=1}^k q_{ij} \log \frac{q_{ij}}{p_{ij}}.
\end{equation}
Additionally, to prevent the trivial solution where all samples collapse into a single cluster, an entropy regularization term $\mathcal{L}_{ent} = -H(\bar{p})$ is introduced, maximizing the entropy of the mean cluster assignment $\bar{p}$ across the batch.

Directly optimizing all terms may lead to suboptimal convergence due to the misalignment of feature spaces. To mitigate this cold-start problem, we implement a progressive curriculum. Specifically, for the initial $T_{warm}$ epochs, the model is optimized solely via $\mathcal{L}_{align}$ to establish a preliminary semantic topology. In the subsequent structure refinement phase ($epoch > T_{warm}$), the full objective $\mathcal{L}_{total}$ is activated. This two-stage mechanism ensures that the cluster centroids $\mathcal{M}$ are initialized within a semantically meaningful manifold, thereby stabilizing the optimization of the decision boundaries. The algorithm is summarized as Algorithm~\ref{alg:tacc_training} in Appendix~\ref{sct:psd}.



\begin{table}[!t]
\caption{Dataset statistics. $m$, $n$, and $k^*$ are the numbers of features, samples, and true number of clusters, respectively.}
\label{tb:statistics}
\centering
\resizebox{0.95\columnwidth}{!}{ 
\begin{tabular}{r|cc|c|ccc} 
\toprule
No. & Dataset & Abbrev. & Domain & $m$ & $n$ & $k^*$ \\
\midrule
1 & Soybean & SO & Agriculture & 35 & 307 & 15 \\	
2 & Breast cancer & BR & Healthcare & 9 & 286 & 2 \\
3 & Car evaluation & CA & Transportation & 6 & 1728 & 4 \\
4 & Lenses & LE & Healthcare & 4 & 24 & 3 \\
5 & Mammographic & MA & Healthcare & 5 & 961 & 2 \\
6 & Zoo & ZO & Biology & 15 & 101 & 7 \\
7 & Hayes-roth & HA & Social Science & 4 & 132 & 3 \\
8 & Fertility & FE & Healthcare & 9 & 100 & 2 \\	
9 & Heart failure & HE & Healthcare & 12 & 299 & 2 \\
10 & Nursery & NU & Social Science & 8 & 12960 & 4 \\
11 & Mushroom & MU & Biology & 21 & 8124 & 2 \\
12 & Adult & AD & Finance & 14 & 48842 & 2 \\	
\bottomrule
\end{tabular}}
\end{table}

\begin{table*}[!t]
\centering
\caption{Clustering performance evaluated by the ACC, NMI, and ARI on various datasets (mean $\pm$ std). The best and second-best results on each dataset are highlighted in \textbf{bold} and \underline{underlined}, respectively. $\Delta$ denotes the relative performance gain of the best method over the second-best, calculated as $(1st - 2nd) / 2nd \times 100\%$. }
\resizebox{2.01\columnwidth}{!}{
\begin{tabular}{l|c|ccccccccc}
\toprule
 \multirow{2}{*}{Data} & \multirow{2}{*}{Index} & \multirow{2}{*}{Kmeans} & {DEC} & {CC} & {SCARF} & {QGRL} & {IDC} & {TabDC} & {TagCC} & \multirow{2}{*}{$\Delta$}  \\      
 & & & [ICLR'16] & [AAAI'21] & [ICLR'22] & [SIGKDD'24] & [ICML'24] & [SIGMOD'25] & [Ours] & \\
\midrule
\multirow{3}{*}{SO} 
& ARI & 0.3918 $\pm$ 0.001 & 0.3808 $\pm$ 0.032 & 0.3050 $\pm$ 0.012 & \cellcolor{blue!10}\underline{0.4015 $\pm$ 0.013} & 0.2620 $\pm$ 0.042 & 0.1489 $\pm$ 0.031 & 0.3340 $\pm$ 0.019 & \cellcolor{red!10}\textbf{0.4773 $\pm$ 0.029} & \textcolor{red}{ $\uparrow$ 18.88\%} \\       
& ACC & \cellcolor{blue!10}\underline{0.5694 $\pm$ 0.001} & 0.5621 $\pm$ 0.039 & 0.4637 $\pm$ 0.020 & 0.5324 $\pm$ 0.019 & 0.3906 $\pm$ 0.031 & 0.2907 $\pm$ 0.045 & 0.4990 $\pm$ 0.018 & \cellcolor{red!10}\textbf{0.6101 $\pm$ 0.034} & \textcolor{red}{ $\uparrow$ 7.15\%} \\
& NMI & 0.6689 $\pm$ 0.001 & 0.6679 $\pm$ 0.024 & 0.5842 $\pm$ 0.013 & \cellcolor{blue!10}\underline{0.6759 $\pm$ 0.009} & 0.4265 $\pm$ 0.049 & 0.3533 $\pm$ 0.047 & 0.6440 $\pm$ 0.012 & \cellcolor{red!10}\textbf{0.7122 $\pm$ 0.015} & \textcolor{red}{ $\uparrow$ 5.37\%} \\        
\midrule
\multirow{3}{*}{BR} 
& ARI & -0.0035 $\pm$ 0.000 & 0.0438 $\pm$ 0.058 & 0.0364 $\pm$ 0.033 & -0.0033 $\pm$ 0.000 & \cellcolor{red!10}\textbf{0.1700 $\pm$ 0.000} & 0.0564 $\pm$ 0.051 & 0.0150 $\pm$ 0.037 & \cellcolor{blue!10}\underline{0.1489 $\pm$ 0.053} & \textcolor{blue}{ $\downarrow$ 12.41\%} \\  
& ACC & 0.5126 $\pm$ 0.000 & 0.5818 $\pm$ 0.083 & 0.5852 $\pm$ 0.051 & 0.5091 $\pm$ 0.003 & \cellcolor{blue!10}\underline{0.7203 $\pm$ 0.000} & 0.6077 $\pm$ 0.063 & 0.5460 $\pm$ 0.068 & \cellcolor{red!10}\textbf{0.7238 $\pm$ 0.015} & \textcolor{red}{ $\uparrow$ 0.49\%} \\
& NMI & 0.0034 $\pm$ 0.000 & 0.0242 $\pm$ 0.028 & 0.0357 $\pm$ 0.028 & 0.0019 $\pm$ 0.001 & \cellcolor{red!10}\textbf{0.0866 $\pm$ 0.010} & 0.0397 $\pm$ 0.027 & 0.0080 $\pm$ 0.011 & \cellcolor{blue!10}\underline{0.0746 $\pm$ 0.026} & \textcolor{blue}{ $\downarrow$ 13.86\%} \\    

\midrule
\multirow{3}{*}{CA} & ARI & \cellcolor{blue!10}\underline{0.0394 $\pm$ 0.000} & 0.0094 $\pm$ 0.013 & 0.0311 $\pm$ 0.012 & 0.0213 $\pm$ 0.015 & -0.0013 $\pm$ 0.005 & 0.0160 $\pm$ 0.018 & 0.0270 $\pm$ 0.029 & \cellcolor{red!10}\textbf{0.1862 $\pm$ 0.034} & \textcolor{red}{ $\uparrow$ 372.59\%} \\ 
& ACC & 0.3595 $\pm$ 0.001 & 0.3659 $\pm$ 0.043 & 0.3401 $\pm$ 0.018 & 0.3418 $\pm$ 0.033 & \cellcolor{red!10}\textbf{0.6241 $\pm$ 0.090} & 0.5036 $\pm$ 0.031 & 0.3390 $\pm$ 0.034 & \cellcolor{blue!10}\underline{0.5198 $\pm$ 0.072} & \textcolor{blue}{ $\downarrow$ 16.71\%} \\
& NMI & \cellcolor{blue!10}\underline{0.0859 $\pm$ 0.000} & 0.0240 $\pm$ 0.021 & 0.0683 $\pm$ 0.015 & 0.0524 $\pm$ 0.029 & 0.0029 $\pm$ 0.002 & 0.0201 $\pm$ 0.015 & 0.0560 $\pm$ 0.027 & \cellcolor{red!10}\textbf{0.2392 $\pm$ 0.074} & \textcolor{red}{ $\uparrow$ 178.46\%} \\      
\midrule
\multirow{3}{*}{LE} 
& ARI & 0.0814 $\pm$ 0.005 & 0.1537 $\pm$ 0.172 & 0.1569 $\pm$ 0.068 & 0.1019 $\pm$ 0.171 & \cellcolor{blue!10}\underline{0.1995 $\pm$ 0.127} & 0.0816 $\pm$ 0.169 & 0.1330 $\pm$ 0.085 & \cellcolor{red!10}\textbf{0.4109 $\pm$ 0.145} & \textcolor{red}{ $\uparrow$ 105.96\%} \\      
& ACC & 0.5333 $\pm$ 0.002 & 0.5667 $\pm$ 0.106 & 0.2875 $\pm$ 0.043 & 0.4958 $\pm$ 0.099 & \cellcolor{blue!10}\underline{0.6667 $\pm$ 0.037} & 0.5333 $\pm$ 0.107 & 0.5290 $\pm$ 0.070 & \cellcolor{red!10}\textbf{0.7042 $\pm$ 0.104} & \textcolor{red}{ $\uparrow$ 5.62\%} \\
& NMI & 0.2927 $\pm$ 0.010 & 0.3177 $\pm$ 0.095 & 0.3254 $\pm$ 0.083 & 0.2006 $\pm$ 0.139 & 0.2681 $\pm$ 0.081 & 0.1981 $\pm$ 0.136 & \cellcolor{blue!10}\underline{0.3550 $\pm$ 0.120} & \cellcolor{red!10}\textbf{0.4191 $\pm$ 0.017} & \textcolor{red}{ $\uparrow$ 18.06\%} \\       
\midrule
\multirow{3}{*}{MA} 
& ARI & 0.3421 $\pm$ 0.000 & 0.3232 $\pm$ 0.022 & 0.3429 $\pm$ 0.002 & 0.2895 $\pm$ 0.055 & \cellcolor{blue!10}\underline{0.3823 $\pm$ 0.013} & 0.3107 $\pm$ 0.044 & 0.3470 $\pm$ 0.007 & \cellcolor{red!10}\textbf{0.4134 $\pm$ 0.016} & \textcolor{red}{ $\uparrow$ 8.13\%} \\    
& ACC & 0.7928 $\pm$ 0.000 & 0.7845 $\pm$ 0.010 & 0.7931 $\pm$ 0.001 & 0.7681 $\pm$ 0.027 & 0.7944 $\pm$ 0.011 & 0.7783 $\pm$ 0.021 & \cellcolor{blue!10}\underline{0.7950 $\pm$ 0.003} & \cellcolor{red!10}\textbf{0.8217 $\pm$ 0.006} & \textcolor{red}{ $\uparrow$ 3.36\%} \\
& NMI & 0.2823 $\pm$ 0.000 & 0.2558 $\pm$ 0.012 & 0.2665 $\pm$ 0.002 & 0.2427 $\pm$ 0.037 & \cellcolor{blue!10}\underline{0.3108 $\pm$ 0.017} & 0.2528 $\pm$ 0.041 & 0.2790 $\pm$ 0.006 & \cellcolor{red!10}\textbf{0.3365 $\pm$ 0.014} & \textcolor{red}{ $\uparrow$ 8.27\%} \\        
\midrule
\multirow{3}{*}{ZO} 
& ARI & 0.6586 $\pm$ 0.013 & 0.6419 $\pm$ 0.078 & 0.3494 $\pm$ 0.056 & 0.6616 $\pm$ 0.088 & 0.7847 $\pm$ 0.039 & 0.4122 $\pm$ 0.155 & \cellcolor{blue!10}\underline{0.8050 $\pm$ 0.059} & \cellcolor{red!10}\textbf{0.8548 $\pm$ 0.089} & \textcolor{red}{ $\uparrow$ 6.19\%} \\   
& ACC & 0.7307 $\pm$ 0.007 & 0.7178 $\pm$ 0.066 & 0.4911 $\pm$ 0.053 & 0.7297 $\pm$ 0.070 & 0.8059 $\pm$ 0.017 & 0.6010 $\pm$ 0.085 & \cellcolor{blue!10}\underline{0.8410 $\pm$ 0.054} & \cellcolor{red!10}\textbf{0.8663 $\pm$ 0.069} & \textcolor{red}{ $\uparrow$ 3.01\%} \\
& NMI & 0.8044 $\pm$ 0.002 & 0.7934 $\pm$ 0.043 & 0.5716 $\pm$ 0.050 & 0.8020 $\pm$ 0.033 & 0.7754 $\pm$ 0.042 & 0.4959 $\pm$ 0.130 & \cellcolor{blue!10}\underline{0.8660 $\pm$ 0.018} & \cellcolor{red!10}\textbf{0.8839 $\pm$ 0.025} & \textcolor{red}{ $\uparrow$ 2.07\%} \\        
\midrule
\multirow{3}{*}{HA} & ARI & 0.0446 $\pm$ 0.001 & 0.0094 $\pm$ 0.013 & 0.0241 $\pm$ 0.022 & \cellcolor{blue!10}\underline{0.0480 $\pm$ 0.035} & -0.0024 $\pm$ 0.006 & 0.0463 $\pm$ 0.036 & 0.0390 $\pm$ 0.018 & \cellcolor{red!10}\textbf{0.0863 $\pm$ 0.066} & \textcolor{red}{ $\uparrow$ 79.79\%} \\  
& ACC & 0.4413 $\pm$ 0.001 & 0.4025 $\pm$ 0.032 & 0.3250 $\pm$ 0.025 & \cellcolor{blue!10}\underline{0.4431 $\pm$ 0.030} & 0.3956 $\pm$ 0.021 & \cellcolor{blue!10}\underline{0.4431 $\pm$ 0.038} & 0.4430 $\pm$ 0.026 & \cellcolor{red!10}\textbf{0.5162 $\pm$ 0.060} & \textcolor{red}{ $\uparrow$ 16.50\%} \\
& NMI & 0.0575 $\pm$ 0.001 & 0.0240 $\pm$ 0.021 & 0.0413 $\pm$ 0.028 & 0.0582 $\pm$ 0.035 & 0.0091 $\pm$ 0.008 & \cellcolor{blue!10}\underline{0.0746 $\pm$ 0.037} & 0.0600 $\pm$ 0.017 & \cellcolor{red!10}\textbf{0.1027 $\pm$ 0.047} & \textcolor{red}{ $\uparrow$ 37.67\%} \\       
\midrule
\multirow{3}{*}{FE} 
& ARI & -0.0041 $\pm$ 0.000 & -0.0098 $\pm$ 0.005 & -0.0038 $\pm$ 0.003 & \cellcolor{blue!10}\underline{0.0302 $\pm$ 0.035} & 0.0171 $\pm$ 0.082 & 0.0033 $\pm$ 0.029 & -0.0080 $\pm$ 0.004 & \cellcolor{red!10}\textbf{0.0637 $\pm$ 0.071} & \textcolor{red}{ $\uparrow$ 110.93\%} \\  
& ACC & 0.5100 $\pm$ 0.000 & 0.5070 $\pm$ 0.009 & 0.5180 $\pm$ 0.010 & 0.6030 $\pm$ 0.075 & \cellcolor{blue!10}\underline{0.7170 $\pm$ 0.108} & 0.5620 $\pm$ 0.073 & 0.5050 $\pm$ 0.005 & \cellcolor{red!10}\textbf{0.7740 $\pm$ 0.103} & \textcolor{red}{ $\uparrow$ 7.95\%} \\
& NMI & 0.0028 $\pm$ 0.000 & 0.0127 $\pm$ 0.008 & 0.0060 $\pm$ 0.005 & 0.0176 $\pm$ 0.014 & \cellcolor{red!10}\textbf{0.0474 $\pm$ 0.027} & 0.0128 $\pm$ 0.015 & 0.0110 $\pm$ 0.009 & \cellcolor{blue!10}\underline{0.0321 $\pm$ 0.034} & \textcolor{blue}{ $\downarrow$ 32.28\%} \\    
\midrule
\multirow{3}{*}{HE} 
& ARI & -0.0033 $\pm$ 0.000 & 0.0438 $\pm$ 0.058 & -0.0014 $\pm$ 0.002 & 0.0273 $\pm$ 0.002 & \cellcolor{blue!10}\underline{0.0829 $\pm$ 0.095} & 0.0036 $\pm$ 0.009 & 0.0090 $\pm$ 0.010 & \cellcolor{red!10}\textbf{0.1563 $\pm$ 0.092} & \textcolor{red}{ $\uparrow$ 88.54\%} \\   
& ACC & 0.5518 $\pm$ 0.000 & 0.5818 $\pm$ 0.083 & 0.5167 $\pm$ 0.012 & 0.6217 $\pm$ 0.003 & \cellcolor{blue!10}\underline{0.6341 $\pm$ 0.081} & 0.5512 $\pm$ 0.021 & 0.5630 $\pm$ 0.031 & \cellcolor{red!10}\textbf{0.7137 $\pm$ 0.033} & \textcolor{red}{ $\uparrow$ 12.55\%} \\
& NMI & 0.0003 $\pm$ 0.000 & 0.0242 $\pm$ 0.028 & 0.0014 $\pm$ 0.001 & 0.0058 $\pm$ 0.001 & \cellcolor{blue!10}\underline{0.0955 $\pm$ 0.048} & 0.0029 $\pm$ 0.003 & 0.0030 $\pm$ 0.002 & \cellcolor{red!10}\textbf{0.1109 $\pm$ 0.076} & \textcolor{red}{ $\uparrow$ 16.13\%} \\       
  
\midrule
\multirow{3}{*}{NU}
& ARI & \cellcolor{blue!10}\underline{0.0927 $\pm$ 0.005} & 0.0651 $\pm$ 0.071 & 0.0839 $\pm$ 0.086 & 0.0900 $\pm$ 0.101 & 0.0800 $\pm$ 0.048 & 0.0103 $\pm$ 0.008 & 0.0430 $\pm$ 0.025 & \cellcolor{red!10}\textbf{0.2822 $\pm$ 0.026} & \textcolor{red}{ $\uparrow$ 204.42\%} \\ 
& ACC & 0.3326 $\pm$ 0.002 & 0.3211 $\pm$ 0.057 & 0.3110 $\pm$ 0.068 & 0.3336 $\pm$ 0.076 & \cellcolor{blue!10}\underline{0.3383 $\pm$ 0.037} & 0.3014 $\pm$ 0.015 & 0.2920 $\pm$ 0.026 & \cellcolor{red!10}\textbf{0.4741 $\pm$ 0.044} & \textcolor{red}{ $\uparrow$ 40.14\%} \\
& NMI & 0.1319 $\pm$ 0.009 & 0.0907 $\pm$ 0.094 & 0.1116 $\pm$ 0.116 & \cellcolor{blue!10}\underline{0.1346 $\pm$ 0.131} & 0.1032 $\pm$ 0.058 & 0.0235 $\pm$ 0.016 & 0.0650 $\pm$ 0.035 & \cellcolor{red!10}\textbf{0.3232 $\pm$ 0.037} & \textcolor{red}{ $\uparrow$ 140.12\%} \\        
\midrule
\multirow{3}{*}{MU} 
& ARI & 0.2067 $\pm$ 0.000 & 0.2026 $\pm$ 0.072 & 0.2567 $\pm$ 0.132 & \cellcolor{blue!10}\underline{0.4464 $\pm$ 0.091} & 0.1974 $\pm$ 0.166 & 0.1932 $\pm$ 0.117 & 0.1260 $\pm$ 0.104 & \cellcolor{red!10}\textbf{0.5266 $\pm$ 0.102} & \textcolor{red}{ $\uparrow$ 17.97\%} \\   
& ACC & 0.7279 $\pm$ 0.000 & 0.7220 $\pm$ 0.041 & 0.7377 $\pm$ 0.088 & \cellcolor{blue!10}\underline{0.8345 $\pm$ 0.042} & 0.7035 $\pm$ 0.098 & 0.7035 $\pm$ 0.086 & 0.6570 $\pm$ 0.094 & \cellcolor{red!10}\textbf{0.8614 $\pm$ 0.035} & \textcolor{red}{ $\uparrow$ 3.22\%} \\
& NMI & 0.1471 $\pm$ 0.000 & 0.1634 $\pm$ 0.050 & 0.2246 $\pm$ 0.123 & \cellcolor{blue!10}\underline{0.4341 $\pm$ 0.108} & 0.1986 $\pm$ 0.147 & 0.1553 $\pm$ 0.084 & 0.1370 $\pm$ 0.038 & \cellcolor{red!10}\textbf{0.4667 $\pm$ 0.100} & \textcolor{red}{ $\uparrow$ 7.51\%} \\         
\midrule
\multirow{3}{*}{AD} 
& ARI & -0.0319 $\pm$ 0.000 & -0.0319 $\pm$ 0.000 & \cellcolor{blue!10}\underline{0.0727 $\pm$ 0.005} & -0.0027 $\pm$ 0.001 &  -0.0422 $\pm$ 0.005 & 0.0150 $\pm$ 0.010 & -0.0320 $\pm$ 0.000 & \cellcolor{red!10}\textbf{0.0953 $\pm$ 0.043} & \textcolor{red}{ $\uparrow$ 31.09\%} \\ 
& ACC & 0.5016 $\pm$ 0.000 & 0.5016 $\pm$ 0.000 & \cellcolor{blue!10}\underline{0.6348 $\pm$ 0.005} & 0.6171 $\pm$ 0.006 & 0.5010 $\pm$ 0.006 & 0.5872 $\pm$ 0.007 & 0.5020 $\pm$ 0.000 & \cellcolor{red!10}\textbf{0.6504 $\pm$ 0.044} & \textcolor{red}{ $\uparrow$ 2.46\%} \\
& NMI & 0.0429 $\pm$ 0.000 & 0.0429 $\pm$ 0.000 & \cellcolor{blue!10}\underline{0.0891 $\pm$ 0.009} & 0.0000 $\pm$ 0.000 & 0.0000 $\pm$ 0.000 & 0.0381 $\pm$ 0.006 & 0.0430 $\pm$ 0.000 & \cellcolor{red!10}\textbf{0.1057 $\pm$ 0.063} & \textcolor{red}{ $\uparrow$ 18.63\%} \\       
\midrule
\multicolumn{2}{c|}{Average Rank} & 5.000 & 5.139 & 5.056 & 4.444 & 4.417 & 5.667 & 5.167 & 1.111 & $-$ \\
\bottomrule
\end{tabular}}
\label{tb:clustering_full_with_gap}
\end{table*}

\section{Experiments}
\label{sct:exp}
To comprehensively validate the effectiveness and interpretability of our proposed framework, the experiments are designed to answer 
five key research questions:

\begin{itemize}
    \item \textbf{Q1: Does TagCC achieve state-of-the-art clustering performance?}
    Section~\ref{sct:performance} compares TagCC with traditional, popular, and state-of-the-art methods on multiple datasets and performs significance testing to confirm the reliability of performance gains.

    \item \textbf{Q2: Does semantic enhancement improve the rationality of representations, and to what extent?}
    Section~\ref{sct:ablation} conducts extensive ablation studies to isolate the effects of the semantic-aware transformation, the dual-branch architecture, and the CL mechanism. 

   \item \textbf{Q3: Does TagCC learn meaningful cluster structures?}
    Section~\ref{sct:visualization} visualizes the learned latent topology via t-SNE to qualitatively examine the structural compactness and separability of the clusters.

    \item \textbf{Q4: Does the optimization process converge reliably and preserve open-world  semantics?} Section~\ref{sct:convergence} analyzes the training dynamics, demonstrating stable convergence and the preservation of semantic consistency throughout the optimization.
    \item \textbf{Q5: Is TagCC robust to parameter and LLM variations?} Appendix~\ref{sct:paanlysis} investigates the parameter sensitivity of the proposed framework and evaluates its robustness when utilizing different LLMs for semantic anchor generation.

\end{itemize}

\subsection{Experimental Settings}\label{sct:exp_Set}

Experimental settings are briefly described below. Please refer to Appendix~\ref{sct:a2des} for the detailed experimental settings.

\paragraph{12 datasets} from various domains are utilized for the experiments, with the number of samples ranging from 0.02k to 48.84k. All the datasets are real public datasets collected from the UCI Machine Learning Repository~\cite{uci}, and their statistical information is shown in Table~\ref{tb:statistics}. Each dataset is preprocessed by removing samples with missing values. For all comparison methods, the number of clusters $k$ is set to $k^*$, i.e., the true number of clusters specified by the dataset labels. 


\begin{table*}[!t]
\centering
\caption{Ablation studies of TagCC. TCC removes the semantic-aware transformation. TLC and TTC denote single-branch deep clustering using only textual or tabular embeddings, respectively. The best results are highlight in \textbf{bold}. $\overline{\textit{AR}}$ reports Average Rank of each method. }
\label{tab:ablation}
\resizebox{1.0\textwidth}{!}{
\begin{tabular}{l|cccc|cccc|cccc}
\toprule
\multirow{2}{*}{Data} & \multicolumn{4}{c}{ARI} & \multicolumn{4}{c}{ACC} & \multicolumn{4}{c}{NMI} \\
\cmidrule(lr){2-5}\cmidrule(lr){6-9}\cmidrule(lr){10-13}
 & TTC & TLC & TCC & TagCC & TTC & TLC & TCC & TagCC & TTC & TLC & TCC & TagCC \\        
\midrule
SO & 0.41 $\pm$ 0.04 & 0.10 $\pm$ 0.02 & 0.38 $\pm$ 0.03 & \textbf{0.48 $\pm$ 0.03} & 0.55 $\pm$ 0.04 & 0.28 $\pm$ 0.03 & 0.51 $\pm$ 0.04 & \textbf{0.61 $\pm$ 0.03} & 0.66 $\pm$ 0.03 & 0.20 $\pm$ 0.04 & 0.53 $\pm$ 0.04 & \textbf{0.71 $\pm$ 0.01} \\
BR & 0.02 $\pm$ 0.05 & 0.02 $\pm$ 0.03 & 0.07 $\pm$ 0.03 & \textbf{0.15 $\pm$ 0.05} & 0.56 $\pm$ 0.06 & 0.59 $\pm$ 0.06 & 0.67 $\pm$ 0.03 & \textbf{0.72 $\pm$ 0.02} & 0.02 $\pm$ 0.03 & 0.01 $\pm$ 0.01 & 0.03 $\pm$ 0.02 & \textbf{0.07 $\pm$ 0.03} \\
CA & 0.05 $\pm$ 0.03 & 0.02 $\pm$ 0.02 & 0.06 $\pm$ 0.05 & \textbf{0.19 $\pm$ 0.03} & 0.37 $\pm$ 0.04 & 0.45 $\pm$ 0.09 & 0.46 $\pm$ 0.12 & \textbf{0.52 $\pm$ 0.07} & 0.07 $\pm$ 0.04 & 0.03 $\pm$ 0.02 & 0.04 $\pm$ 0.02 & \textbf{0.24 $\pm$ 0.07} \\
LE & 0.08 $\pm$ 0.12 & 0.12 $\pm$ 0.13 & 0.21 $\pm$ 0.18 & \textbf{0.41 $\pm$ 0.15} & 0.50 $\pm$ 0.09 & 0.57 $\pm$ 0.07 & 0.62 $\pm$ 0.11 & \textbf{0.70 $\pm$ 0.10} & 0.22 $\pm$ 0.18 & 0.16 $\pm$ 0.10 & 0.27 $\pm$ 0.13 & \textbf{0.42 $\pm$ 0.02} \\
MA & 0.38 $\pm$ 0.02 & 0.24 $\pm$ 0.09 & 0.40 $\pm$ 0.01 & \textbf{0.41 $\pm$ 0.02} & 0.81 $\pm$ 0.01 & 0.74 $\pm$ 0.05 & 0.80 $\pm$ 0.00 & \textbf{0.82 $\pm$ 0.01} & 0.32 $\pm$ 0.02 & 0.19 $\pm$ 0.07 & 0.31 $\pm$ 0.01 & \textbf{0.34 $\pm$ 0.01} \\
ZO & 0.68 $\pm$ 0.07 & 0.51 $\pm$ 0.03 & 0.80 $\pm$ 0.06 & \textbf{0.85 $\pm$ 0.09} & 0.75 $\pm$ 0.04 & 0.69 $\pm$ 0.04 & 0.84 $\pm$ 0.05 & \textbf{0.87 $\pm$ 0.07} & 0.80 $\pm$ 0.03 & 0.59 $\pm$ 0.06 & 0.85 $\pm$ 0.03 & \textbf{0.88 $\pm$ 0.03} \\
HA & 0.04 $\pm$ 0.03 & 0.00 $\pm$ 0.00 & 0.05 $\pm$ 0.02 & \textbf{0.09 $\pm$ 0.07} & 0.46 $\pm$ 0.04 & 0.00 $\pm$ 0.00 & 0.47 $\pm$ 0.04 & \textbf{0.52 $\pm$ 0.06} & 0.05 $\pm$ 0.02 & 0.00 $\pm$ 0.00 & 0.05 $\pm$ 0.02 & \textbf{0.10 $\pm$ 0.05} \\
FE & -0.00 $\pm$ 0.01 & \textbf{0.09 $\pm$ 0.01} & 0.09 $\pm$ 0.03 & 0.06 $\pm$ 0.07 & 0.54 $\pm$ 0.03 & 0.70 $\pm$ 0.02 & 0.75 $\pm$ 0.07 & \textbf{0.77 $\pm$ 0.10} & 0.00 $\pm$ 0.01 & \textbf{0.05 $\pm$ 0.01} & 0.04 $\pm$ 0.02 & 0.03 $\pm$ 0.03 \\
HE & 0.01 $\pm$ 0.02 & 0.03 $\pm$ 0.02 & 0.05 $\pm$ 0.02 & \textbf{0.16 $\pm$ 0.09} & 0.56 $\pm$ 0.04 & 0.59 $\pm$ 0.02 & 0.63 $\pm$ 0.03 & \textbf{0.71 $\pm$ 0.03} & 0.01 $\pm$ 0.01 & 0.02 $\pm$ 0.01 & 0.04 $\pm$ 0.02 & \textbf{0.11 $\pm$ 0.08} \\
NU & 0.05 $\pm$ 0.04 & 0.12 $\pm$ 0.08 & 0.27 $\pm$ 0.04 & \textbf{0.28 $\pm$ 0.03} & 0.32 $\pm$ 0.05 & 0.46 $\pm$ 0.06 & 0.47 $\pm$ 0.04 & \textbf{0.47 $\pm$ 0.04} & 0.06 $\pm$ 0.07 & 0.12 $\pm$ 0.08 & 0.31 $\pm$ 0.05 & \textbf{0.32 $\pm$ 0.04} \\
MU & 0.35 $\pm$ 0.19 & 0.21 $\pm$ 0.16 & 0.28 $\pm$ 0.12 & \textbf{0.53 $\pm$ 0.10} & 0.78 $\pm$ 0.10 & 0.70 $\pm$ 0.11 & 0.76 $\pm$ 0.07 & \textbf{0.86 $\pm$ 0.04} & 0.32 $\pm$ 0.17 & 0.16 $\pm$ 0.13 & 0.21 $\pm$ 0.09 & \textbf{0.47 $\pm$ 0.10} \\
AD & 0.04 $\pm$ 0.04 & 0.03 $\pm$ 0.04 & 0.04 $\pm$ 0.04 & \textbf{0.10 $\pm$ 0.04} & 0.58 $\pm$ 0.04 & 0.54 $\pm$ 0.04 & 0.60 $\pm$ 0.04 & \textbf{0.65 $\pm$ 0.04} & 0.04 $\pm$ 0.05 & 0.03 $\pm$ 0.05 & 0.04 $\pm$ 0.05 & \textbf{0.11 $\pm$ 0.06} \\
\midrule
$\overline{\textit{AR}}$ & 3.25 & 3.42 & 2.17 & 1.17 & 3.25 & 3.50 & 2.25 & 1.00 & 2.92 & 3.58 & 2.33 & 1.17 \\
\bottomrule
\end{tabular}}
\end{table*}

\vspace{-0.5cm}

\paragraph{Seven representative methods}are selected for comparison. $k$-means~\cite{kms} and DEC~\cite{DEC} are adopted as traditional and generative baselines. Contrastive Clustering (CC)~\cite{CC} and SCARF~\cite{scarf} are included to represent CL paradigms. Moreover, Quaternion Graph Representation Learning (QGRL)~\cite{QGRL}, Interpretable Deep Clustering (IDC)~\cite{IDC}, and TableDC~\cite{tabledc} are employed as the latest state-of-the-art (SOTA) counterparts. Each method is executed ten times, and the average performance is reported.

\paragraph{Three validity indices}including clustering ACCuracy (ACC)~\cite{ex1}, Normalized Mutual Information (NMI)~\cite{ex5}, and Adjusted Rand Index (ARI)~\cite{ex3,ex4}, are adopted to evaluate clustering performance, where larger values indicate
better performance. We also conduct the two-tailed Bonferroni-Dunn (BD) post-hoc test~\cite{significance-test} to statistically analyze the superiority of our TagCC.

\subsection{Clustering Performance}
\label{sct:performance}

Clustering performances of different methods across ACC, NMI, and ARI indices are reported in Table \ref{tb:clustering_full_with_gap}. The best and second-best results on each dataset are highlighted in \textbf{bold} and \underline{underlined}, respectively. $\Delta$ denotes the relative performance gain of the best method over the second-best, calculated as $(1st - 2nd) / 2nd \times 100\%$.

The observations include the following four aspects: 1) TagCC achieves the best performance on 11 out of 12 datasets in terms of ARI and ACC, and 10 datasets in terms of NMI. This consistent performance across all three metrics confirms its effectiveness compared to its counterparts; 2) On the BR and AD datasets, TagCC is not obviously better than the second-best method. However, the second-best method varies across different datasets, highlighting the robustness of TagCC against varying data distributions; 3) Although QGRL obtains a higher ACC on the CA dataset, its NMI and ARI scores are extremely low. This is because QGRL tends to assign most samples to the majority class, whereas TagCC avoids such degeneracy by anchoring tabular feature to distinct prototypes, prevents the representations from collapsing into a single cluster; 4) On datasets with relatively limited sample sizes, e.g., LE, ZO, HA, and FE, TagCC maintains a clear performance advantage over the competing methods, validating the effectiveness of semantic enhancement when statistical information is sparse. 

Furthermore, significance testing via the two-tailed Bonferroni-Dunn post-hoc test confirms that TagCC yields statistically significant improvements over all competing counterparts at the 95\% confidence level. For detailed test statistics and the Critical
Difference (CD) diagram, please refer to Appendix~\ref{sct:sig_study}.


\subsection{Ablation Studies}
\label{sct:ablation}

To demonstrate the effectiveness of the core components of TagCC, several ablated variants are compared as shown in Table~\ref{tab:ablation}. The best results are highlight in \textbf{bold}. To evaluate the semantic-aware transformation, TagCC is modified into TCC, which replaces the LLM-based semantic generation with concatenation-based serialization. Furthermore, to highlight the necessity of the dual-branch contrastive alignment, the TCC variant is decoupled into the single-branch llm-only TLC and table-only TTC models by conducting deep clustering exclusively on the isolated textual descriptions and tabular representations, respectively. 

As shown in Table~\ref{tab:ablation}, TagCC consistently outperforms the TCC variant on 11 out of 12 datasets, verifying that the proposed semantic-aware transformation effectively unveils open-world knowledge compared to the verbose and generic descriptions yielded by concatenation-based serialization. Furthermore, TCC generally surpasses the single-branch baselines, i.e., TLC and TTC, on 8 out of 12 datasets, demonstrating that the dual-branch contrastive mechanism successfully harmonizes heterogeneous modalities to learn robust representations even without optimal semantic guidance. In contrast, TLC and TTC exhibit comparable but suboptimal performance, highlighting their inherent deficiencies, i.e., the former lacks structural constraints, while the latter lacks explicit semantic guidance, relying instead on brittle statistical co-occurrence to infer relationships. Moreover, the results for TLC on the HA dataset are not reported because the concatenation-based serialization of abstract codes yields indistinguishable textual representations, resulting in clustering failure. This phenomenon further underscores the critical importance of our semantic-aware mechanism in ensuring semantic discriminability.

\subsection{Qualitative Results}
\label{sct:visualization}

\begin{figure}[!t]	
\centerline{\includegraphics[width=3.2in]{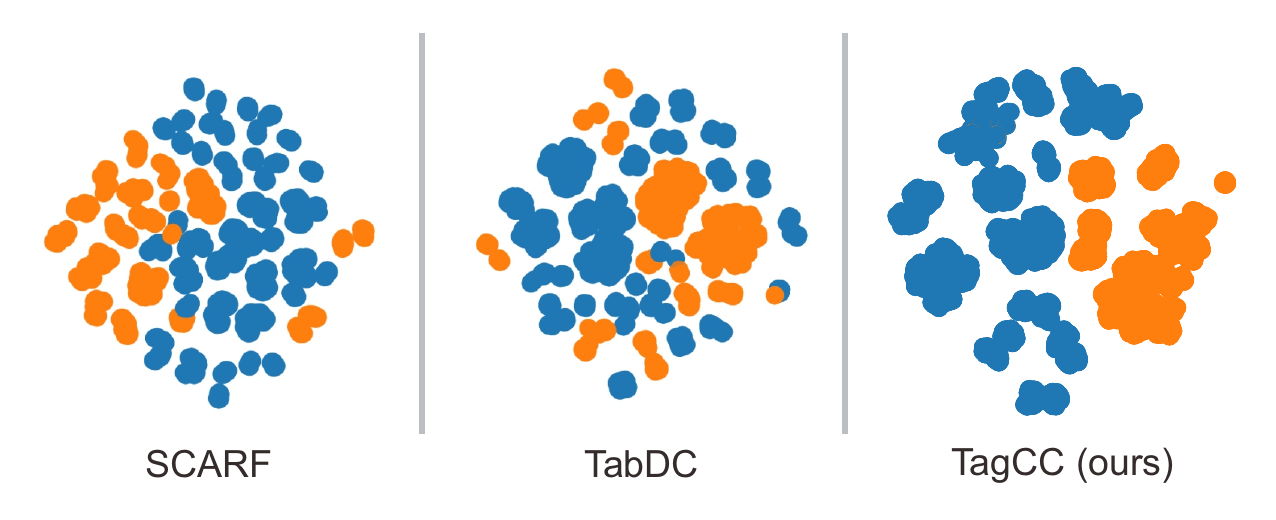}}
\caption{t-SNE visualization of representations learned by SCARF, TabDC, and TagCC on the MU dataset. Orange and blue points mark the samples with the ``true'' cluster labels.}
\label{fig:tsne}	
\end{figure}

To demonstrate the cluster discrimination ability of the proposed TagCC method, we use SCARF, TabDC, and TagCC to extract latent feature representations for the MU dataset, which are then projected into a 2-D space for visualization via t-SNE~\cite{r1visual}, as shown in Figure~\ref{fig:tsne}. Orange and blue points mark the samples with the ``true'' cluster labels. This visualizes how well each method captures the distinguishable cluster distributions and preserves the intrinsic local structure for intuitive comparison.

It can be seen from Figure~\ref{fig:tsne} that TagCC obviously has better cluster discrimination ability compared to the counterparts. SCARF and TabDC exhibit blurred boundaries and considerable overlap between different categories; conversely, TagCC forms highly compact and well-separated clusters with clear margins. This indicates that TagCC performs representation learning more clustering-friendly, leveraging semantic constraints to disentangle the feature space and better suit the exploration of the underlying $k^*$ clusters.

\subsection{Convergence Analysis}
\label{sct:convergence}

Figure~\ref{fig:convergence} illustrates the convergence curves of TagCC on the CA (left) and MM (right) datasets. To simultaneously visualize the high-magnitude alignment loss and the low-magnitude prototype loss, we employ a broken-axis scale in the visualization. The training process is explicitly divided into two phases: the initial 50 epochs (indicated to the left of the vertical gray dashed line) constitute the semantic warm-up stage, followed by the structure refinement stage. Specifically, in the semantic warm-up stage, the model is optimized exclusively by the alignment loss $\mathcal{L}_{align}$ to capture global semantic dependencies. To subsequently induce compact cluster structures within this semantic manifold, the objective in the structure refinement stage, the objective transitions transitions to a sum $\mathcal{L}_{total} = \mathcal{L}_{align} + \mathcal{L}_{proto}$.

During the semantic warm-up stage, the alignment loss $\mathcal{L}_{align}$ decreases rapidly, establishing a strong semantic consistency between tabular representations and the semantic anchors. A critical observation is the behavior of the prototype loss $\mathcal{L}_{proto}$ upon its activation instructure refinement stage. As shown in the bottom section of the plots, $\mathcal{L}_{proto}$ initializes at a remarkably low magnitude rather than a high random value. This empirically verifies that our semantic warm-up strategy has successfully pre-structured the latent space, providing a high-quality, low-entropy initialization for the clustering module.

\begin{figure}[!t]	
\centerline{\includegraphics[width=3.3in]{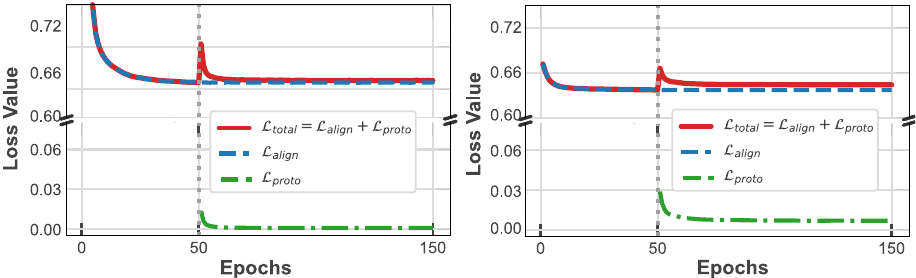}}
\caption{Convergence analysis on CA (left) and MA (right). The vertical dashed line at epoch 50 separates the semantic warm-up stage from the structure refinement stage.}
\label{fig:convergence}	
\end{figure}

In the structure refinement stage, $\mathcal{L}_{align}$ transitions from rapid descent to a stable plateau. This stability indicates that the alignment loss functions as a semantic regularizer, anchoring the learnable representations to the semantic manifold to prevent catastrophic forgetting while the model adapts to the data distribution. Concurrently, the continuous decrease in $\mathcal{L}_{proto}$ corresponds to the fine-grained refinement of decision boundaries, specifically resolving ambiguity for hard samples near cluster margins. The smooth trajectory of the total loss demonstrates the robustness of our joint optimization scheme, successfully balancing semantic preservation with structural adaptation.

\section{Concluding Remark}
\label{sct:cr}

In this paper, we proposed TagCC, a novel framework designed to bridge the gap between statistical tabular learning and open-world semantics. To address the data sparsity inherent in discrete tabular data and the lack of discriminative focus in concatenation-based serialization, TagCC introduces a semantic anchoring paradigm. Specifically, by leveraging LLMs to distill intrinsic data semantics into structured and discriminative textual anchors, and explicitly aligning tabular representations with these anchors via CL, our framework effectively enriches tabular representations with open-world semantics, while mitigating the risk of overfitting to isolated statistical correlations. Moreover, this alignment is jointly optimized with a clustering objective, which yields partitions that are both semantically coherent and structurally compact. Extensive experiments on multiple datasets have demonstrated that TagCC consistently outperforms its counterparts. In the future, it is promising to extend TagCC into more complex scenarios, e.g., handling tabular data with missing values, adapting it to non-stationary data distributions, or integrating privacy-preserving mechanisms for sensitive domains.

\bibliography{TaCC}
\bibliographystyle{icml2026}

\appendix

\onecolumn

\end{document}